\documentclass{article}

\usepackage{hyperref}
\usepackage{url}
\usepackage{booktabs}       
\usepackage{amsfonts}       
\usepackage{nicefrac}       
\usepackage{microtype}      

\usepackage{amsmath,amsfonts,amsthm}       
\usepackage{bm}
\usepackage{bbm}
\usepackage{multirow}
\usepackage[inline]{enumitem}
\usepackage{diagbox}
\usepackage[capitalise]{cleveref}  
\usepackage{comment}
\usepackage{etoolbox}
\usepackage{graphicx}
\usepackage{wrapfig}
\usepackage[font=small]{caption}
\usepackage{algorithm,algorithmicx,algpseudocode}
\usepackage{subcaption}
\usepackage[usenames,dvipsnames]{xcolor}         

\usepackage{pifont}

\usepackage{import}
\graphicspath{{figures/}{../figures/}}
\usepackage{subfiles}

\newtoggle{comments}
\togglefalse{comments}
\iftoggle{comments}{
    \newcommand\AG[1]{\textcolor{magenta}{[AG: #1]}}
    \newcommand\EN[1]{\textcolor{cyan}{[EN: #1]}}
    \newcommand\KG[1]{\textcolor{blue}{[KG: #1]}}
    \newcommand\GD[1]{\textcolor{green}{[GD: #1]}}
}{
    \newcommand\AG[1]{}
    \newcommand\EN[1]{}
    \newcommand\KG[1]{}
    \newcommand\GD[1]{}
}

\newtheorem{theorem}{Theorem}

\newtheorem{lemma}{Lemma}[section]
\newtheorem{corollary}[lemma]{Corollary}

\newtheorem{definition}{Definition}

\newcommand{\dt}{\Delta}

\newcommand{\method}{S4ND}

\newtoggle{arxiv}
\toggletrue{arxiv}

\newcommand{\para}[1]{\iftoggle{arxiv}{\paragraph{#1}}{\paragraph{#1}}} 

\iftoggle{arxiv}{
  \usepackage[square,numbers,sort]{natbib}
  \setlength{\textwidth}{6.5in}
  \setlength{\textheight}{9in}
  \setlength{\oddsidemargin}{0in}
  \setlength{\evensidemargin}{0in}
  \setlength{\topmargin}{-0.5in}
  \newlength{\defbaselineskip}
  \setlength{\defbaselineskip}{\baselineskip}
  \setlength{\marginparwidth}{0.8in}
  \setlength{\parskip}{5pt}%
  \setlength{\parindent}{0pt}%
  \usepackage[dvipsnames]{xcolor}         
}{
  \PassOptionsToPackage{numbers,square,compress}{natbib}
  \usepackage{neurips_2022}
}

\title{S4ND: Modeling Images and Videos as Multidimensional Signals Using State Spaces}

\author{%
  Eric Nguyen\thanks{Equal contribution.}$^\ast$$^\dagger$, Karan Goel$^\ast$$^\ddagger$, Albert Gu$^\ast$$^\ddagger$,\\
  Gordon W. Downs$^\ddagger$, Preey Shah$^\ddagger$, Tri Dao$^\ddagger$, Stephen A. Baccus$^\mathsection$, Christopher Ré$^\ddagger$\\
  $^\dagger$Department of BioEngineering, Stanford University\\
  $^\ddagger$Department of Computer Science, Stanford University\\
  $^\mathsection$Department of Neurobiology, Stanford University\\
  {\small\texttt{\{etnguyen,albertgu,gwdowns,preey,trid,baccus}\}\texttt{@stanford.edu}}\\
  {\small\texttt{\{kgoel,chrismre}\}\texttt{@cs.stanford.edu}}\\
}

\begin{document}

\maketitle

\begin{abstract}

Visual data such as images and videos are typically modeled as discretizations of inherently continuous, multidimensional signals. 
Existing continuous-signal models attempt to exploit this fact by modeling the underlying signals of visual (e.g., image) data directly.
However, these models have not yet been able to achieve competitive performance on practical vision tasks such as large-scale image and video classification.
Building on a recent line of work on deep state space models (SSMs), we propose \method, a new multidimensional SSM layer that extends the continuous-signal modeling ability of SSMs to multidimensional data including images and videos.
We show that S4ND can model large-scale visual data in $1$D, $2$D, and $3$D as continuous multidimensional signals and demonstrates strong performance by simply swapping Conv2D and self-attention layers with \method\ layers in existing state-of-the-art models.
On ImageNet-1k, \method\ exceeds the performance of a Vision Transformer baseline by $1.5\%$ when training with a $1$D sequence of patches, and matches ConvNeXt when modeling images in $2$D. For videos, S4ND improves on an inflated $3$D ConvNeXt in activity classification on HMDB-51 by $4\%$.
S4ND implicitly learns global, continuous convolutional kernels that are resolution invariant by construction, providing an inductive bias that enables generalization across multiple resolutions.
By developing a simple bandlimiting modification to S4 to overcome aliasing, S4ND achieves strong zero-shot (unseen at training time) resolution performance, outperforming a baseline Conv2D by $40\%$ on CIFAR-10 when trained on $8 \times 8$ and tested on $32 \times 32$ images.
%
When trained with progressive resizing, S4ND comes within $\sim 1\%$ of a high-resolution model while training $22\%$ faster.

\end{abstract}

\section{Introduction}
Modeling visual data such as images and videos is a canonical problem in deep learning. 
In the last few years, many modern deep learning backbones that achieve strong performance on benchmarks like ImageNet~\citep{Russakovsky2015ImageNetLS} have been proposed. 
%
These backbones are diverse, and include $1$D sequence models such as the Vision Transformer (ViT)~\citep{dosovitskiy2020image}, which treats images as sequences of patches, and $2$D and $3$D models that use local convolutions over images and videos (ConvNets)~\citep{Krizhevsky2012ImageNetCW, he2016deep,simonyan2014very, szegedy2015going, tan2021efficientnetv2,Liu2022ACF,hara2017learning,ji20123d,qiu2017learning,tran2015learning,feichtenhofer2019slowfast}.

A commonality among modern vision models capable of achieving state-of-the-art (SotA) performance is that they treat visual data as discrete pixels rather than continuous-signals. 
However, images and videos are discretizations of multidimensional and naturally {continuous} signals, sampled at a fixed rate in the spatial and temporal dimensions. 
Ideally, we would want approaches that are capable of recognizing this distinction between data and signal, and directly model the underlying continuous-signals. 
This would give them capabilities like the ability to adapt the model to data sampled at different resolutions.

\begin{figure}[!t]
    \centering
    \includegraphics[width=\linewidth]{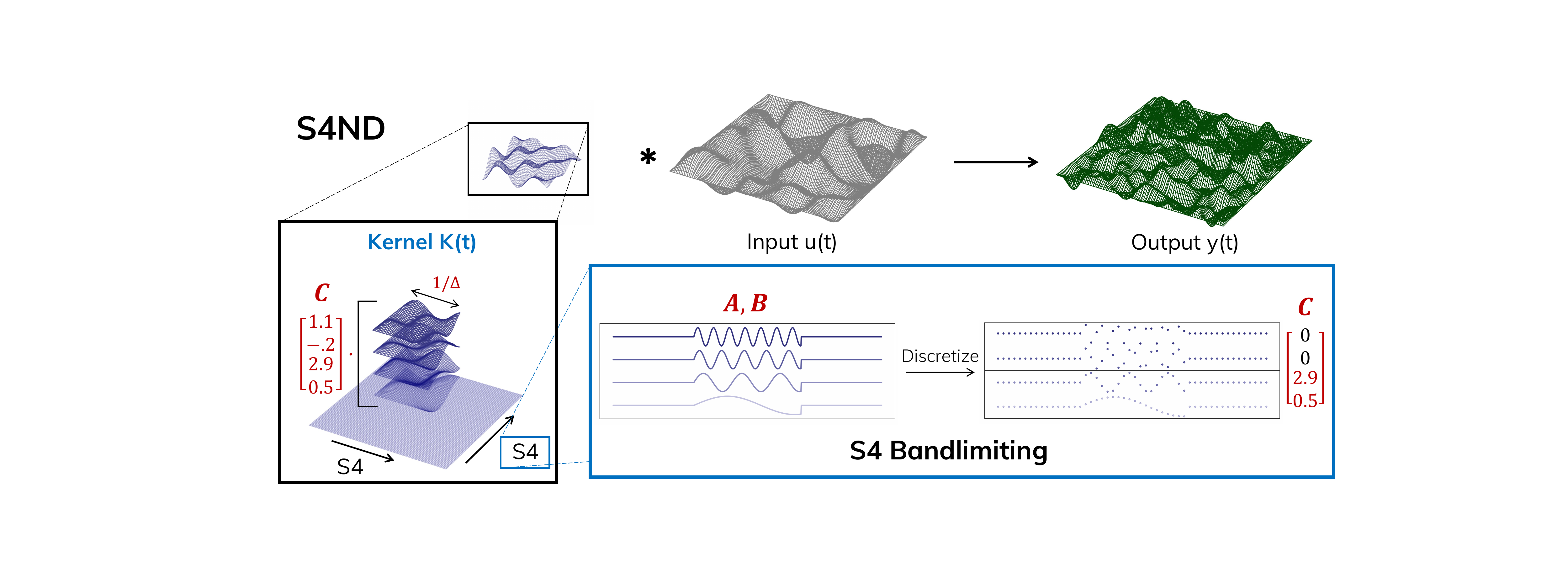}
    \caption{
      (\textbf{S4ND}.) (Parameters in \emph{red}.)
      (\emph{Top}) S4ND can be viewed as a depthwise convolution that maps a multidimensional input (\emph{black}) to output (\emph{green}) through a continuous convolution kernel (\emph{blue}).
      (\emph{Bottom Left}) The kernel can be interpreted as a linear combination (controlled by \( \bm{C} \)) of basis functions (controlled by \( \bm{A}, \bm{B} \)) with flexible width (controlled by step size \( \dt \)).
      For structured \( \bm{C} \), the kernel can further factored as a low-rank tensor product of $1$D kernels, and can be interpreted as independent S4 transformations on each dimension.
      (\emph{Bottom Right}) Choosing \( \bm{A}, \bm{B} \) appropriately yields Fourier basis functions with controllable frequencies. To avoid aliasing in the final discrete kernels, the coefficients of \( \bm{C} \) corresponding to high frequencies can simply be masked out.
    }
\label{fig:s4nd}
\end{figure}

A natural approach to building such models is to parameterize and learn continuous convolutional kernels, which can then be sampled differently for data at different resolutions~\citep{Finzi2020GeneralizingCN,Romero2021FlexConvCK,Schtt2017SchNetAC,gu2021lssl,gu2022efficiently}.
Among these, deep state space models (SSM)~\citep{gu2021lssl}, in particular S4~\citep{gu2022efficiently}, have achieved SotA results in modeling sequence data derived from continuous-signals, such as audio~\citep{goel2022sashimi}. 
However, a key limitation of SSMs is that they were developed for $1$D signals, and cannot directly be applied to visual data derived from multidimensional ``ND" signals. 
%
%
Given that $1$D SSMs outperform other continuous modeling solutions for sequence data~\citep{gu2022efficiently}, and have had preliminary success on image~\citep{gu2022efficiently} and video classification~\citep{islam2022long}, we hypothesize that they may be well suited to modeling visual data when appropriately generalized to the setting of multidimensional signals.

Our main contribution is \method, a new deep learning layer that extends S4 to multidimensional signals. 
The key idea is to turn the standard SSM (a $1$D ODE) into a multidimensional PDE governed by an independent SSM per dimension.
By adding additional structure to this ND SSM, we show that it is equivalent to an ND continuous convolution that can be factored
into a separate $1$D SSM convolution per dimension.
This results in a model that is efficient and easy to implement, using the standard $1$D S4 layer as a black box. Furthermore, it can be controlled by S4's parameterization, allowing it to model both long-range dependencies, or finite windows with a learnable window size that generalize conventional local convolutions \citep{gu2022train}.%

We show that \method{} can be used as a drop-in replacement in strong modern vision architectures while matching or improving performance in $1$D, $2$D, and $3$D.
With minimal change to the training procedure, replacing the self-attention in ViT with S4-1D improves top-$1$ accuracy by $1.5\%$, and replacing the convolution layers in a $2$D ConvNeXt backbone \citep{Liu2022ACF} with S4-2D preserves its performance on ImageNet-1k \citep{deng2009imagenet}.
Simply inflating (temporally) this pretrained S4-2D-ConvNeXt backbone to $3$D improves video activity classification results on HMDB-51 \citep{kuehne2011hmdb} by $4$ points over the pretrained ConvNeXt baseline. 
Notably, we use S4ND as global kernels that span the entire input shape, which enable it to have global context (both spatially and temporally) in every layer of a network.

Additionally, we propose a low-pass bandlimiting modification to S4 that encourages the learned convolutional kernels to be smooth. 
While \method\ can be used at any resolution, performance suffers when moving between resolutions due to aliasing artifacts in the kernel, an issue also noted by prior work on continuous models~\citep{Romero2021FlexConvCK}. 
While S4 was capable of transferring between different resolutions on audio data~\citep{gu2022efficiently},
visual data presents a greater challenge due to the scale-invariant properties of images in space and time \citep{ruderman1993statistics}, as sampled images with more distant objects are more likely to contain power at frequencies above the Nyquist cutoff frequency. Motivated by this, we propose a simple criteria that masks out frequencies in the \method\ kernel that lie above the Nyquist cutoff frequency.

\looseness=-1
The continuous-signal modeling capabilities of \method\ open the door to new training recipes, such as the ability to train and test at different resolutions.
On the standard CIFAR-10 \citep{krizhevsky2009cifar} and Celeb-A \citep{liu2015image} datasets, \method{} degrades by as little as $1.3\%$ when upsampling from low- to high-resolution data (e.g. $128\times 128 \to 160\times 160$), and can be used to facilitate progressive resizing to speed up training by $22\%$ with $\sim 1\%$ drop in final accuracy compared to training at the high resolution alone. 
We also validate that our new bandlimiting method is critical to these capabilities, with ablations showing absolute performance degradation of up to $20\%+$ without it.

\vspace{-3pt}
\section{Related Work}
\vspace{-5pt}

{\bf Image Classification.} There is a long line of work in image classification, with much of the 2010s dominated by ConvNet backbones~\citep{Krizhevsky2012ImageNetCW, he2016deep,simonyan2014very, szegedy2015going, tan2021efficientnetv2}. 
Recently, Transformer backbones, such as ViT \citep{dosovitskiy2020image}, have achieved SotA performance on images using self-attention over a sequence of $1$D patches \citep{liu2021swin, liu2021swin2, touvron2021deit, xiaohua2021scaling, ding2022davit}. 
Their scaling behavior in both model and dataset training size is believed to give them an inherent advantage over ConvNets \citep{dosovitskiy2020image}, even with minimal inductive bias. 
%
\citet{Liu2022ACF} introduce ConvNeXt, which modernizes the standard ResNet architecture~\citep{he2016deep} using modern training techniques, matching the performance of Transformers on image classification. 
We select a backbone in the $1$D and $2$D settings, ViT and ConvNeXt, to convert into continuous-signal models by replacing the multi-headed self-attention layers in ViT and the standard Conv2D layers in ConvNeXt with \method\ layers, maintaining their top-1 accuracy on large-scale image classification.

{\bf S4 \& Video Classification.} 
To handle the long-range dependancies inherent in videos, \citep{islam2022long} used 1D S4 for video classification on the Long-form Video Understanding dataset \citep{wu2021lvu}.
They first applied a Transformer to each frame to obtain a sequence of patch embeddings for each video frame independently, followed by a standard 1D S4 to model across the concatenated sequence of patches. This is akin to previous methods that learned spatial and temporal information separately \citep{karpathy2014large}, for example using ConvNets on single frames, followed by an LSTM \cite{hochreiter1997lstm} to aggregate temporal information. In contrast, modern video architectures such as $3$D ConvNets and Transformers \citep{hara2017learning,ji20123d,qiu2017learning,tran2015learning,feichtenhofer2019slowfast,kondratyuk2021movinets,liu2021videoswin,wu2021lvu,arnab2021vivit,akbari2021vatt} show stronger results when learning spatiotemporal features simultaneously, which the generalization of \method\ into multidimensions now enables us to do.

{\bf Continuous-signal Models.} Visual data are discretizations of naturally continuous signals that possess extensive structure in the joint distribution of spatial frequencies, including the properties of scale and translation invariance. For example, an object in an image generates correlations between lower and higher frequencies that arises in part from phase alignment at edges \citep{olshausen1996natural}. As an object changes distances in the image, these correlations remain the same but the frequencies shift. This relationship can potentially be learned from a coarsely sampled image and then applied at higher frequency at higher resolution.

A number of continuous-signal models have been proposed for the visual domain to learn these inductive biases, and have led to additional desirable properties and capabilities.
A classic example of continuous-signal driven processing is the fast Fourier transform, which is routinely used for filtering and data consistency in computational and medical imaging \citep{desai2021vortex}.
NeRF represents a static scene as a continuous function, allowing them to render scenes smoothly from multiple viewpoints \citep{mildenhall2020nerf}. 
CKConv \citep{romero2021ckconv} learns a continuous representation to create kernels of arbitrary size for several data types including images,
with additional benefits such as
the ability to handle irregularly sampled data. 
FlexConv \citep{Romero2021FlexConvCK} extends this work with a learned kernel size,
and show that images can be trained at low resolution and tested at high resolution if the aliasing problem is addressed.
%
S4 \citep{gu2022efficiently} increased abilities to model long-range dependancies using continuous kernels, allowing SSMs to achieve SotA on sequential CIFAR \citep{krizhevsky2009cifar}.
However, these methods including 1D S4 have been applied to relatively low dimensional data, e.g., time series, and small image datasets. \method\ is the first continuous-signal model applied to high dimensional visual data with the ability to maintain SotA performance on large-scale image and video classification. 





{\bf Progressive Resizing.} Training times for large-scale image classification can be quite long, a trend that is exacerbated by the emergence of foundation models~\citep{bommasani2021foundation}.
A number of strategies have emerged for reducing overall training time.
Fix-Res~\citep{touvron2019fixing} trains entirely at a lower resolution, and then fine-tunes at the higher test resolution to speed up training in a two-stage process.
Mix-and-Match~\citep{hoffer2019mix} randomly samples low and high resolutions during training in an interleaved manner.
An effective method to reduce training time on images is to utilize progressive resizing. This involves training at a lower resolution and gradually upsampling in stages.
%
For instance, \citet{howard2018fastai} utilized progressive resizing to train an ImageNet in under $4$ hours.
EfficientNetV2 \citep{tan2021efficientnetv2} coupled resizing with a progressively regularization schedule, increasing the regularization as well to maintain accuracy. In EfficientNetV2 and other described approaches, the models eventually train on the final test resolution.
As a continuous-signal model, we demonstrate that \method\ is naturally suited to progressive resizing, while being able to generalize to \textit{unseen} resolutions at test time.

\section{Preliminaries}
\label{sec:background}

\paragraph{State space models.}
S4 investigated state space models,
which are linear time-invariant systems that map signals \( u(t) \mapsto y(t) \) and can be represented either as a linear ODE (equation \eqref{eq:ssm})
or convolution (equation \eqref{eq:ssm-convolution}).
Its parameters are \( \bm{A} \in \mathbbm{C}^{N \times N} \) and \( \bm{B}, \bm{C} \in \mathbbm{C}^{N} \) for a state size \( N \).



\begin{minipage}{.5\linewidth}
\begin{equation}
  \label{eq:ssm}
\begin{aligned}
  x'(t) &= \bm{A}x(t) + \bm{B}u(t)  \\
  y(t)  &= \bm{C}x(t)
\end{aligned}
\end{equation}
\end{minipage}
\begin{minipage}{.5\linewidth}
\begin{equation}%
  \label{eq:ssm-convolution}
  \begin{aligned}
    K(t) &= \bm{C} e^{t\bm{A}} \bm{B} \\
    y(t) &= (K \ast u)(t) 
  \end{aligned}
\end{equation}
\end{minipage}


\paragraph{Basis functions.}

For the clearest intuition, we think of the convolution kernel as a linear combination (controlled by \( \bm{C} \)) of \textbf{basis kernels} \( K_n(t) \) (controlled by \( \bm{A}, \bm{B} \))
\begin{align}
  \label{eq:ssm-basis}
  K(t) = \sum_{k=0}^{N-1} \bm{C}_k K_k(t)
  \qquad \qquad
  K_n(t) = (e^{t\bm{A}}\bm{B})_n
\end{align}

\para{Discretization.}

\AG{TODO replace this with a discussion of width of kernel}

The SSM \eqref{eq:ssm} is defined over a continuous-time axis
and produces continuous-time convolution kernels \eqref{eq:ssm-convolution}\eqref{eq:ssm-basis}.
Given a discrete input sequence \( u_0, u_1, \dots \) sampled uniformly from an underlying signal \( u(t) \)  at a step size \( \dt \) (i.e. \( u_k = u(k\dt) \)),
the kernel can be sampled to match the rate of the input.
Note that instead of directly sampling the kernel, standard discretization rules should be applied to minimize the error from the discrete to the continuous-time kernel~\citep{gu2022efficiently}.
For inputs given at different resolutions, the model can then simply change its \( \dt \) value to compute the kernel at different resolutions.

We note that the step size \( \dt \) does not have to be exactly equal to a ``true sampling rate'' of the underlying signal, but only the relative rate matters.
Concretely, the discrete-time kernel depends only on the \emph{product} \( \dt \bm{A} \) and \( \dt \bm{B} \), and S4 learns separate parameters \( \dt, \bm{A}, \bm{B} \).
\AG{maybe cut out the above}\KG{this is not clear from context}

\para{S4.}

S4 is a special SSM with prescribed \( (\bm{A}, \bm{B}) \) matrices that define well-behaved basis functions, and an algorithm that allows the convolution kernel to be computed efficiently.
Variants of S4 exist that define different basis functions,
such as simple diagonal SSMs \citep{gupta2022diagonal},
or one that defines \textbf{truncated Fourier functions} \( K_n(t) = \sin(2\pi n t) \mathbbm{I}[(0, 1)] \) \citep{gu2022train} (\cref{fig:s4nd}).
These versions of S4 have easy-to-interpret basis functions that will allow us to control the frequencies in the kernel (\cref{sec:method:bandlimit}).



\section{Method}



We describe the proposed S4ND model for the 2D case only, for ease of notation and presentation.
The results extend readily to general dimensions;
full statements and proofs for the general case are in \cref{sec:theory-details}.
\cref{sec:method:s4nd} describes the multidimensional S4ND layer, 
and \cref{sec:method:bandlimit} describes our simple modification to restrict frequencies in the kernels.
\cref{fig:s4nd} illustrates the complete S4ND layer.

\subsection{S4ND}
\label{sec:method:s4nd}

We begin by generalizing the (linear time-invariant) SSM \eqref{eq:ssm} to higher dimensions.
Notationally, we denote the individual time axes with superscripts in parentheses.
Let \( u = u(t^{(1)}, t^{(2)}) \) and \( y = y(t^{(1)}, t^{(2)}) \) be the input and output which are signals \( \mathbbm{R}^2 \to \mathbbm{C} \),
and \( x = (x^{(1)}(t^{(1)}, t^{(2)}), x^{(2)}(t^{(1)}, t^{(2)})) \in \mathbbm{C}^{N^{(1)} \times N^{(2)}} \) be the SSM state of dimension \( N^{(1)} \times N^{(2)} \), where \( x^{(\tau)} : \mathbbm{R}^2 \to \mathbbm{C}^{N^{(\tau)}} \).
\begin{definition}[Multidimensional SSM]%
  \label{def:nd-ssm}
  Given parameters \( \bm{A}^{(\tau)} \in \mathbbm{C}^{N^{(\tau)} \times N^{(\tau)}} \), \( \bm{B}^{(\tau)} \in \mathbbm{C}^{N^{(\tau)} \times 1} \),
  \( \bm{C} \in \mathbbm{C}^{N^{(1)} \times N^{(2)}} \),
  the 2D SSM is the map \( u \mapsto y \) defined by the linear PDE with initial condition \( x(0, 0) = 0 : \)
  \begin{equation}
    \label{eq:nd-ssm}
    \begin{aligned}
      \frac{\partial}{\partial t^{(1)}} x(t^{(1)}, t^{(2)}) &= (\bm{A}^{(1)} x^{(1)}(t^{(1)}, t^{(2)}), x^{(2)}(t^{(1)}, t^{(2)})) + \bm{B}^{(1)} u(t^{(1)}, t^{(2)})
      \\
      \frac{\partial}{\partial t^{(2)}} x(t^{(1)}, t^{(2)}) &= (x^{(1)}(t^{(1)}, t^{(2)}), \bm{A}^{(2)} x^{(2)}(t^{(1)}, t^{(2)})) + \bm{B}^{(2)} u(t^{(1)}, t^{(2)})
      \\
      y(t^{(1)}, t^{(2)}) &= \langle \bm{C}, x(t^{(1)}, t^{(2)}) \rangle
    \end{aligned}
  \end{equation}
\end{definition}
Note that \cref{def:nd-ssm} differs from the usual notion of multidimensional SSM, which is simply a map from \( u(t) \in \mathbbm{C}^n \mapsto y(t) \in \mathbbm{C}^m \) for higher-dimensional \( n, m > 1 \) but still with 1 time axis.
However, \cref{def:nd-ssm} is a map from \( u(t_1, t_2) \in \mathbbm{C}^1 \mapsto y(t_1, t_2) \in \mathbbm{C}^1 \) for \emph{scalar} input/outputs but over \emph{multiple} time axes.
When thinking of the input \( u(t^{(1)}, t^{(2)}) \) as a function over a 2D grid,
\cref{def:nd-ssm} can be thought of as a simple linear PDE that just runs a standard 1D SSM
over each axis independently.

Analogous to equation \eqref{eq:ssm-convolution},
the 2D SSM can also be viewed as a multidimensional convolution.
\begin{theorem}%
  \label{thm:nd-ssm}
  \eqref{eq:nd-ssm} is a time-invariant system that is equivalent to a 2D convolution \( y = K \ast u \) by the kernel
  \begin{align}
    \label{eq:nd-kernel}
    K(t^{(1)}, t^{(2)}) &= \langle \bm{C}, (e^{t^{(1)}\bm{A^{(1)}}} \bm{B^{(1)}}) \otimes (e^{t^{(2)}\bm{A^{(2)}}} \bm{B^{(2)}}) \rangle
  \end{align}
  This kernel is a linear combination of the \( N^{(1)} \times N^{(2)} \) basis kernels
  \( \{ K^{(1)}_{n^{(1)}}(t^{(1)}) \otimes K^{(1)}_{n^{(2)}}(t^{(2)}) : n^{(1)} \in [N^{(1)}], n^{(2)} \in [N^{(2)}] \} \)
  where \( K^{(\tau)} \) are the standard 1D SSM kernels \eqref{eq:ssm-basis} for each axis .
\end{theorem}

However, a limitation of this general form is that the number of basis functions \( N^{(1)} \times N^{(2)} \times \dots \)
grows exponentially in the dimension, increasing the parameter count (of \( \bm{C} \)) and overall computation dramatically.
This can be mitigated by factoring \( \bm{C} \) as a low-rank tensor.

\begin{corollary}%
  \label{cor:nd-ssm-factored}
  Suppose that \( \bm{C} \in \mathbbm{C}^{N^{(1)} \times N^{(2)}} \) is a low-rank tensor
  \(
    \bm{C} = \sum_{i=1}^r \bm{C_i}^{(1)} \otimes \bm{C_i}^{(2)}
  \)
  where each \( \bm{C^{(\tau)}}_i \in \mathbbm{C}^{N^{(\tau)}} \).
  Then the kernel \eqref{eq:nd-kernel} also factors as a tensor product of 1D kernels
  \begin{align*}
    K(t^{(1)}, t^{(2)}) &= \sum_{i=1}^r K^{(1)}_i(t^{(1)}) \otimes K^{(2)}_i(t^{(2)}) := \sum_{i=1}^r (\bm{C_i^{(1)}} e^{t^{(2)}\bm{A^{(1)}}} \bm{B^{(1)}}) \otimes ( \bm{C_i^{(2)}} e^{t^{(2)}\bm{A^{(2)}}} \bm{B^{(2)}})
  \end{align*}
\end{corollary}

In our experiments, we choose \( \bm{C} \) as a rank-1 tensor, but the rank can be freely adjusted to tradeoff parameters and computation for expressivity.
Using the equivalence between \eqref{eq:ssm} and \eqref{eq:ssm-convolution},
\cref{cor:nd-ssm-factored} also has the simple interpretion as defining an independent 1D SSM along each axis of the multidimensional input.

\subsection{Resolution Change and Bandlimiting}
\label{sec:method:bandlimit}

SSMs in 1D have shown strong performance in the audio domain,
and can nearly preserve full accuracy when tested zero-shot on inputs sampled at very different frequencies \citep{gu2022efficiently}.
This capability relies simply on scaling \( \dt \) by the relative change in frequencies
(i.e., if the input resolution is doubled, halve the SSM's \( \dt \) parameter).
However, sampling rates in the spatial domain are often much lower than temporally,
leading to potential aliasing when changing resolutions.
A standard technique to avoid aliasing is to apply a low-pass filter to remove frequencies above the Nyquist cutoff frequency.

For example, when \( \bm{A} \) is diagonal with \( n \)-th element \( \bm{a}_n \), each basis function has simple form
\( K_n(t) = e^{t \bm{a}_n}\bm{B}_n \).
Note that the frequencies are mainly controlled by the imaginary part of \( \bm{a}_n \).
We propose the following simple method:
for any \( n \) such that
\( \bm{a}_n \cdot \dt < \frac{1}{2}\alpha \),
mask out the corresponding coefficient of the linear combination \( \bm{C}_n \) (equation \eqref{eq:ssm-basis}).
Here \( \alpha \) is a hyperparameter that controls the cutoff;
theoretically, \( \alpha=1.0 \) corresponds to the Nyquist cutoff if the basis functions are pure sinusoids.
However, due to the decay \( e^{\Re(\bm{a}_n)} \) arising from the real part
as well as approximations arising from using finite-state SSMs,
\( \alpha \) often has to be set lower empirically.

\newcommand{\lowres}{$\mathrm{low}$}
\newcommand{\midres}{$\mathrm{mid}$}
\newcommand{\baseres}{$\mathrm{base}$}
\newcommand{\tinyres}{$\mathrm{tiny}$}

\section{Experiments}
\label{sec:experiments}

We evaluate \method\ on large-scale image classification in \cref{sec:image-classification} in the $1$D and $2$D settings, followed by activity classification in videos in \cref{sec:video-classification} in the $3$D setting,
where using \method{} as a drop-in replacement for standard deep learning layers matches or improves performance in all settings.
In \cref{sec:continuous-signal}, we performed controlled ablations to highlight the benefits of \method\ as a continuous-signal model in images.

\begin{figure}[!t]
    \centering
    \includegraphics[width=\linewidth]{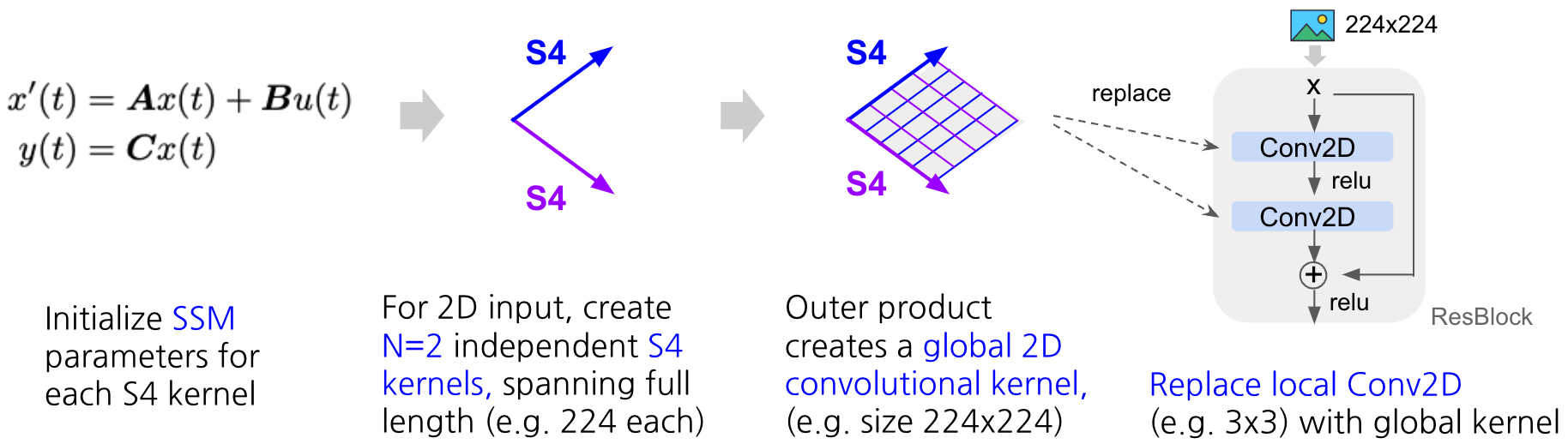}
    \caption{
      (\textbf{Flowchart of S4ND for images: 2D example}.) S4ND can process images as $2$D inputs by initializing an SSM per spatial dimension $x$ and $y$ of the input. Two independent S4 kernels are then instantiated that span the entire input lengths of each dimension (e.g., 224 as shown above). Computing an outer product of the two $1$D kernels produces a global convolutional kernel (e.g., 224x224). This global kernel can replace standard local Conv2D layers where ever they are used, such as ResNet or ConvNeXt blocks. A similar procedure can be done in $3$D (with 3 S4 kernels) to create $3$D global kernels for videos. 
    }
\label{fig:pipeline}
\end{figure}

\subsection{S4ND in 1D \& 2D: Large-scale Image Classification}
\label{sec:image-classification}

First, we show that \method\ is a drop-in replacement for existing visual modeling layers such as $1$D self-attention and $2$D local convolutions, with no degradation in top-$1$ performance when used in modern backbones such as ViT~\citep{dosovitskiy2020image} and ConvNeXt~\citep{Liu2022ACF} on ImageNet-1k~\citep{deng2009imagenet}.

\begin{wraptable}[16]{R}{0.56\textwidth}  
      \small
      \caption{{(\bf Performance on image classification.)}
        Top-1 test accuracy benchmarks for images in the $1$D and $2$D settings. ConvNeXt-M, for ``micro", is a reduced model size for Celeb-A, while ``-ISO" is an isotropic S4ND backbone~\citep{gu2022efficiently}.
      }
    \vspace{3mm}
    \centering
        \begin{tabular}{llcc}
            \toprule
            \textsc{Model} & \textsc{Dataset} & \textsc{Params} & \textsc{Acc}\\
            \midrule
            ViT-B & ImageNet & 88.0M & 78.9 \\
            S4ND-ViT-B & ImageNet & 88.8M & \textbf{80.4} \\
            \midrule
            ConvNeXt-T & ImageNet & 28.4M & 82.1 \\
            S4ND-ConvNeXt-T & ImageNet & 30.0M & \textbf{82.2} \\
            \midrule
            Conv2D-ISO & CIFAR-10 & 2.2M & 93.7 \\
            S4ND-ISO & CIFAR-10 & 5.3M & \textbf{94.1} \\
            \midrule
            ConvNeXt-M & Celeb-A & 9.2M & 91.0 \\
            S4ND-ConvNeXt-M & Celeb-A & 9.6M & \textbf{91.3} \\
            \bottomrule
        \end{tabular}
        \label{tab:classification1}
\end{wraptable}

\para{Baselines and Methodology.} 
We consider large-scale image classification on the ImageNet-1k dataset, which consists of $1000$ classes and $1.3$M images.
We start with two strong baselines: ViT-B (\baseres, $88$M parameters) for processing images in the $1$D setting and ConvNeXt-T (\tinyres, $28.4$M) in the $2$D setting. (We omit the postfix ``B" and ``T" for brevity). 
More recent works using Transformers on images have surpassed ViT, but we focus on the original ViT model to highlight specifically the drop-in capability and performance difference in self-attention vs. \method\ layers. 
We first swap the self-attention layers in ViT with S4ND layers, and call this model S4ND-ViT. 
Notably, we simplify ViT by removing the positional encodings, as \method\ does not require injecting this inductive bias. Similarly, we swap the local Conv2D layers in the ConvNeXt blocks with S4ND layers, which we call S4ND-ConvNeXt, a model with global context at each layer.
%
%
%
Both \method\ variants result in similar parameter counts compared to their baseline models.

\para{Training.} For all ImageNet models, we train from scratch with no outside data and adopt the training procedure from \citep{touvron2021deit, yuan2021tokens}, which uses the AdamW optimizer \citep{loshchilov2017decoupled} for 300 epochs, cosine decay learning rate, weight decay 0.05, and aggressive data augmentations including RandAugment \citep{cubuk2020randaugment}, Mixup \citep{zhang2017mixup}, and AugMix \citep{hendrycks2019augmix}. We add RepeatAug \citep{repeataug} for ConvNeXt and \method-ConvNeXt. The initial learning rate for ViT (and S4ND-ViT) is $0.001$, while for ConvNeXt (and S4ND-ConvNeXt) it is $0.004$. See \cref{appendix:expt-details-image} for additional training procedure details.

\para{Results.} Table \ref{tab:classification1} shows top-$1$ accuracy results for each model on ImageNet. After reproducing the baselines, S4ND-ViT was able to moderately boost performance by +$1.5$\% over ViT, while S4ND-ConvNeXt matched the original ConvNeXt's performance. This indicates that S4ND is a strong primitive that can replace self-attention and standard $2$D convolutions in practical image settings with large-scale data.

\subsection{S4ND in 3D: Video Classification}
\label{sec:video-classification}

Next, we demonstrate the flexible capabilities of \method{} in settings involving pretraining and even higher-dimensional signals.
We use the activity recognition dataset HMDB-51 \citep{kuehne2011hmdb} which involves classifying videos in $51$ activity classes.

\para{Baselines and Methodology.}
Prior work demonstrated that $2$D CNNs (e.g. pretrained on ImageNet) can be adapted to $3$D models by $2$D to $3$D kernel \emph{inflation} (I3D \citep{carreira2017i3d}), in which the 2D kernels are repeated temporally $N$ times and normalized by $1/N$. Our baseline, which we call ConvNeXt-I3D, uses the $2$D ConvNeXt pretrained on ImageNet (\cref{sec:image-classification}) with I3D inflation. Notably, utilizing \method{} in $3$D enables global context \emph{temporally} as well. We additionally test the more modern spatial-temporal separated $3$D convolution used by S3D \citep{xie2017s3d} and R(2+1)D \citep{tran2018r21}, which factor the $3$D convolution kernel as the outer product of a $2$D (spatial) by $1$D (temporal) kernel. Because of its flexible factored form (\cref{sec:method:s4nd}), \method{} automatically has these inflation capabilities. We inflate the pretrained \method{}-ConvNeXt simply by loading the pretrained $2$D model weights for the spatial dimensions, and initializing the temporal kernel parameters \( \bm{A}^{(3)}, \bm{B}^{(3)}, \bm{C}^{(3)} \) from scratch. We note that this model is essentially identical to the baseline ConvNeXt-S3D except that each component of the factored kernels use standard $1$D S4 layers instead of 1D local convolutions. Finally, by varying the initialization of these parameters, we can investigate additional factors affecting model training;
in particular, we also run an ablation on the kernel timescales \( \dt \).


\para{Training.}
Our training procedure is minimal, using only RGB frames (no optical flow).
We sample clips of $2$ seconds with $30$ total frames at $224 \times 224$, followed by applying RandAugment; we performed a small sweep of the RandAugment magnitude for each model.
All models are trained with learning rate $0.0001$ and weight decay \( 0.2 \). Additional details are included in \cref{appendix:expt-details-video}.

\begin{table}[ht]
\vspace{-1mm} 
      \caption{
        (\textbf{HMDB-51 Activity Recognition with ImageNet-pretrained models.})
        (\emph{Left}) Top-1 accuracy with $2$D to $3$SD kernel inflation.
        (\emph{Right}) Ablation of initial temporal kernel lengths, controlled by S4's $\dt$ parameter.
      }
\vspace{1mm}
\hspace{.1\linewidth}
    \begin{minipage}[c]{0.5\textwidth}
      \small
        \centering
        \begin{tabular}{@{}llll@{}}
            \toprule
            & \textsc{Params} & \textsc{Flow} & \textsc{RGB} \\
            \midrule
            Inception-I3D & 25.0M & 61.9 & 49.8 \\
            \midrule
            ConvNeXt-I3D & 28.5M & - & 58.1 \\
            ConvNeXt-S3D & 27.9M & - & 58.6 \\
            S4ND-ConvNeXt-3D & 31.4M & - & \textbf{62.1} \\
            \bottomrule
        \end{tabular}
    \end{minipage}
\hspace{.01\linewidth}
\begin{minipage}[c]{0.3\textwidth}
        \small
        \centering
        \begin{tabular}{@{}ll@{}}
            \toprule
            \textsc{Init. Length} & \textsc{Acc} \\
            \midrule
            20.0 & 53.74 \\
            4.0 & 58.33 \\
            2.0 & 60.30 \\
            1.0 & 62.07 \\
            \bottomrule
        \end{tabular}
    \end{minipage}
        \label{tab:hmdb}
\end{table}

\para{Results.}

Results are presented in \cref{tab:hmdb}.
Our baselines are much stronger than prior work in this setting,
$8\%$ top-$1$ accuracy higher than the original I3D model in the RGB frames only setting,
and confirming that separable kernels (S3D) perform at least as well as standard inflation ($+0.53\%$). \method{}-ConvNeXt-3D improves over the baseline ConvNeXt-I3D by $+4.0\%$ with no difference in models other than using a temporal S4 kernel. This even exceeds the performance of I3D when trained on optical flow.

Finally, we show how \method 's parameters can control for factors such as the kernel length (\cref{tab:hmdb}).
Note that our temporal kernels $K^3(t^3)$ are always full length (30 frames in this case), while standard convolution kernels are shorter temporally and require setting the width of each layer manually as a hyperparameter~\citep{xie2017s3d}.
S4 layers have a parameter $\dt$ that can be interpreted such that $\frac{1}{\dt}$ is the expected length of the kernel.
By simply adjusting this hyperparameter, \method\ can be essentially initialized with length-$1$ temporal kernels that can automatically learn to cover the whole temporal length if needed.
We hypothesize that this contributes to \method{}'s improved performance over baselines.

\begin{table}[b]
    \centering
    \caption{({\bf Settings for continuous capabilities experiments.)} Datasets and resolutions used for continuous capabilities experiments, as well as the model backbones used are summarized.}
    \vspace{1.5mm}
    \begin{tabular}{@{}llllll@{}}
    \toprule
        \textsc{Dataset} & \textsc{Classes} & \multicolumn{3}{c}{\textsc{Resolution}}  & \textsc{Backbone} \\
        \cmidrule(lr){3-5}
       & & \baseres & \midres & \lowres & \\
       \midrule
       CIFAR-10 & $10$ & $32 \times 32$ & $16 \times 16$ $(2\times)$ & $8 \times 8$ $(4\times)$ & Isotropic\\
       Celeb-A & $40$ multilabel & $160 \times 160$ & $128 \times 128$ $(1.25\times)$ & $64 \times 64$ $(2.50\times)$ & ConvNeXt\\
       \bottomrule
    \end{tabular}
    \label{tab:dataset-resolutions}
\end{table}

\begin{table}[t]
    \centering
    \caption{{\bf Zero-Shot Resolution Change.} Results for models trained on one resolution (one of \lowres\ / \midres\ / \baseres), and zero-shot tested on another. Results are averaged over 2 random seeds. \KG{comment out stdevs for this table for submission}}
    \vspace{1.5mm}
    \resizebox{\columnwidth}{!}{
    \begin{tabular}{ccccccc}
         \toprule
         \multicolumn{2}{c}{\textsc{Resolution}} & \multicolumn{3}{c}{\textsc{CIFAR-10}} & \multicolumn{2}{c}{\textsc{Celeb-A}} \\
         \cmidrule(lr){1-2} \cmidrule(lr){3-5} \cmidrule(lr){6-7}
         \textsc{Train}    & \textsc{Test}     & \textsc{\method}         & \textsc{Conv2D}         & \textsc{FlexNet-16}     & \textsc{\method}         & \textsc{Conv2D} \\
         \midrule
         \baseres & \baseres & $\mathbf{93.10 \pm 0.22}$ & $91.9 \pm 0.2$ & $92.2 \pm 0.1$ & $\mathbf{91.75 \pm 0.00}$ & $91.44 \pm 0.03$ \\
         \midrule
         \midres  & \midres  & $\mathbf{88.80 \pm 0.12}$ & $87.2 \pm 0.1$ & $86.5 \pm 2.0$ & $\mathbf{91.63 \pm 0.04}$ & $91.09 \pm 0.08$ \\
         \midres  & \baseres & $\mathbf{88.77 \pm 0.03}$ & $73.1 \pm 0.3$ & $82.7 \pm 2.0$ & $\mathbf{90.14 \pm 0.38}$ & $80.52 \pm 0.08$ \\
         \midrule
         \lowres  & \lowres  & $\mathbf{78.17 \pm 0.13}$ & $76.0 \pm 0.2$ & -              & $\mathbf{90.95 \pm 0.02}$ & $90.37 \pm 0.04$ \\
         \lowres  & \midres  & $\mathbf{78.86 \pm 0.22}$ & $57.4 \pm 0.3$ & -              & $\mathbf{84.44 \pm 1.04}$ & $80.45 \pm 0.11$ \\
         \lowres  & \baseres & $\mathbf{73.71 \pm 0.47}$ & $33.1 \pm 1.3$ & -              & $\mathbf{84.73 \pm 0.54}$ & $80.59 \pm 0.14$ \\
         \bottomrule
    \end{tabular}
    }
    \label{tab:zero-shot-resolution-change-table}
\end{table}


\subsection{Continuous-signal Capabilities for Images}
\label{sec:continuous-signal}

Images are often collected at varied resolutions, even with the same hardware, so it is desirable that models generalize to data sampled at different resolutions. 
We show that \method\ inherits this capability as a continuous-signal model, with strong zero-shot performance when changing resolutions, and the ability to train with progressively resized multi-resolution data.
We perform an ablation to show that our proposed bandlimiting modification is critical to achieving strong performance when changing resolutions.

\textbf{Setup.}
We focus on image classification on $2$D benchmark datasets: a dataset with low-resolution images (CIFAR-10) and one with higher-resolution images (Celeb-A). 
For each dataset, we specify a base image resolution (\baseres), and two lower resolutions (\midres, \lowres), summarized in \cref{tab:dataset-resolutions}.
To highlight that \method 's continuous capabilities are independent of backbone, we experiment with $2$ different $2$D backbones: an isotropic, fixed-width model backbone on CIFAR-10 \citep{krizhevsky2009cifar} and a small ConvNeXt backbone on Celeb-A \cite{liu2015image}.
For each backbone, we compare \method 's performance to Conv2D layers as a standard, widely used baseline. Additional details can be found in \cref{appendix:expt-details-continuous}.

We first verify that \method\ achieves comparable test classification performance with baseline Conv2D models. \cref{tab:zero-shot-resolution-change-table}  and \cref{fig:cifar-zeroshot-chart} show that we exceed the performance of Conv2D on both tasks, with \method\ improving over Conv2D models by $0.4\%$ on CIFAR-10 and $0.3\%$ on Celeb-A. 


{\bf Zero-Shot Resolution Change.} We train \method\ and Conv2D models at either \lowres\ or \midres\ resolution, and test them at the \baseres\ resolution for each dataset. 
We also compare to FlexConv \citep{Romero2021FlexConvCK} on CIFAR-10, which is the current SotA for zero-shot resolution change.
Compared to training and testing at the \baseres\ resolution, we expect that Conv2D should degrade more strongly than \method, since it cannot adapt its kernel appropriately to the changed resolution. 
\cref{tab:zero-shot-resolution-change-table} and \cref{fig:cifar-zeroshot-chart} show that \method\ outperforms Conv2D on \midres\ $\rightarrow$ \baseres\ by $15+$ points and \lowres\ $\rightarrow$ \baseres\ by $40+$ points on CIFAR-10, and $9+$ points and $3+$ points on Celeb-A. 
In fact, \method\ \newline

\begin{wrapfigure}[22]{rt}{0.5\textwidth}  
  \includegraphics[width=0.5\textwidth]{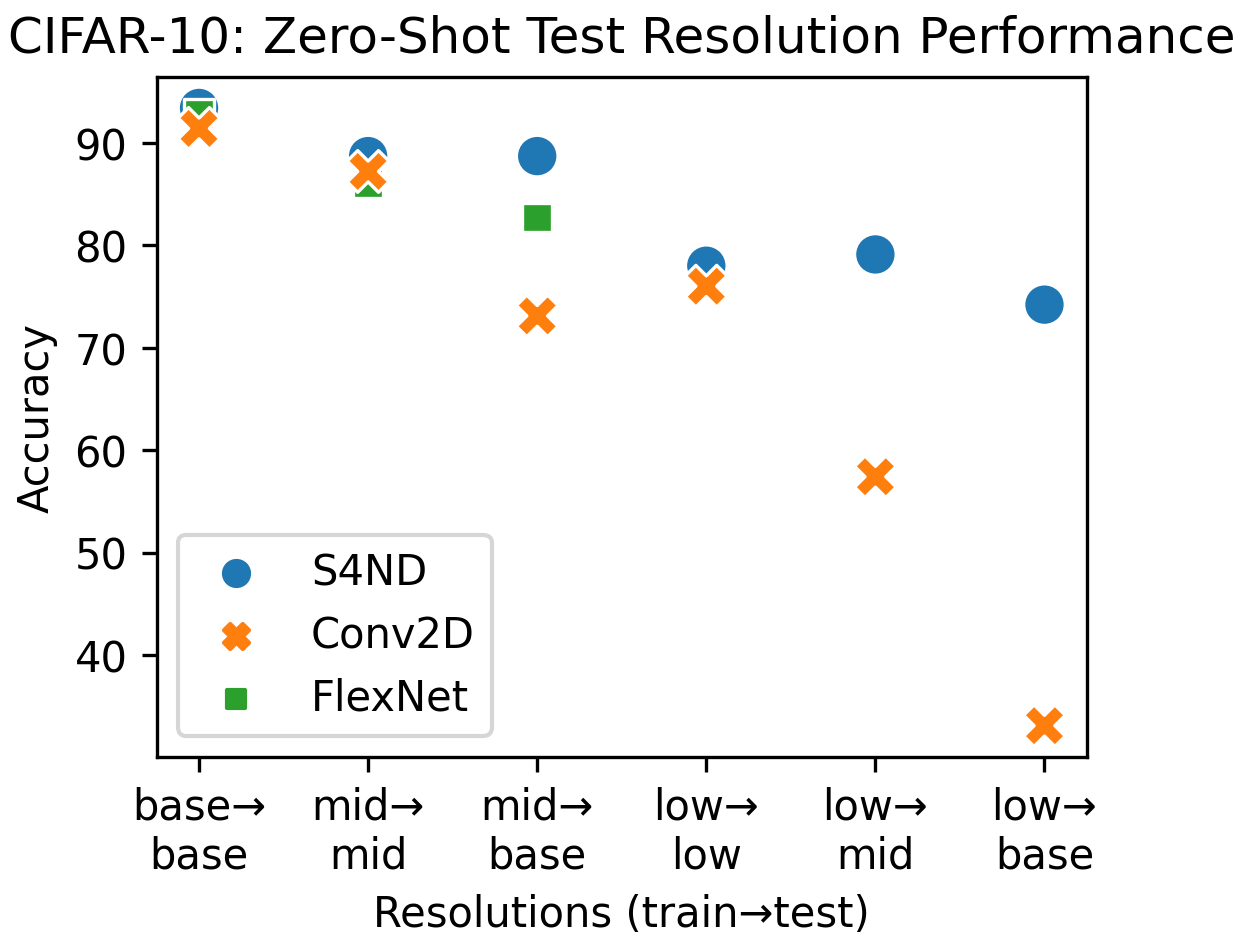}
  \captionsetup{type=figure}
  \caption{
      (\textbf{CIFAR-10 zero-shot comparison.}) When trained and tested on the same resolutions, all models have similar performance (with S4ND slightly better). But, when trained and tested on different resolutions (the zero-shot setting), S4ND significantly outperforms Conv2D and FlexNet.
  }
      \label{fig:cifar-zeroshot-chart}
\end{wrapfigure}

yields better performance on \lowres\ $\rightarrow$ \baseres\ (a more difficult task) than Conv2D does from \midres\ $\rightarrow$ \baseres\ (an easier task), improving by $1.1\%$ on CIFAR-10 and $3\%$ on Celeb-A.
Compared to FlexConv on CIFAR-10 \midres\ $\rightarrow$ \baseres, \method\ improves zero-shot performance by $5+$ points, setting a new SotA.

\para{Progressively Resized Training.} 
We provide an exploration of training with progressive resizing~\citep{howard2018fastai, tan2021efficientnetv2} i.e. training in multiple stages at different resolutions. The only change we make from standard training is to reset the learning rate scheduler at the beginning of each stage (details in \cref{appendix:expt-details-continuous}).
We compare \method\ and the Conv2D baseline with progressive resizing in \cref{tab:progressive-resizing-results}.

For CIFAR-10, we train with a \lowres\ $\rightarrow$ \baseres, $80-20$ epoch schedule, and perform within $\sim 1\%$ of an \method\ model trained with \baseres\ resolution data while speeding up training by $21.8\%$.
We note that Conv2D attains much higher speedups as a consequence of highly optimized implementations, which we discuss in more detail in~\cref{sec:discussion}. 
For Celeb-A, we explore flexibly combining the benefits of both progressive resizing and zero-shot testing, training with a \lowres\ $\rightarrow$ \midres, $16-4$ epoch schedule that uses no \baseres\ data. We outperform Conv2D by $7.5\%+$, and attain large speedups of $50\%+$ over training at the \baseres\ resolution.

\begin{table}[ht!]
  \small
  \caption{
    (\textbf{Progressive resizing results.}) Validation performance for progressively resized training at \baseres\ resolution, and speedup compared to training at \baseres\ resolution on CIFAR-10 and Celeb-A. We use a $80 - 20$ and $16-4$ schedule for CIFAR-10 and Celeb-A, and also report performance training only at \baseres\ resolution.
  }
  \vspace{1.5mm}
  \centering
  \resizebox{\linewidth}{!}{
  \begin{tabular}{@{}llllll@{}}
    \toprule
    \textsc{Dataset} & \textsc{Model} & \textsc{Epoch Schedule} & \textsc{Train Resolution} & \textsc{Val @ Base Res.} & \textsc{Speedup (Step Time)} \\
    \midrule
    CIFAR-10 & Conv2D & -       & \baseres                         & \( 91.90\% \)          & $0\%$    \\
             & S4ND   & -       & \baseres                         & \( \mathbf{93.40}\% \) & $0\%$    \\
             \cmidrule{2-6}
             & Conv2D & $80-20$ & \lowres \(\rightarrow\) \baseres & \( 90.94\% \)          & $51.7\%$ \\
             & S4ND   & $80-20$ & \lowres \(\rightarrow\) \baseres & \( \mathbf{92.32}\% \) & $21.8\%$ \\
    \midrule
    CelebA   & Conv2D & -       & \baseres                         & \( 91.44\% \)          & $0\%$    \\
             & S4ND   & -       & \baseres                         & \( \mathbf{91.75}\% \) & $0\%$    \\
             \cmidrule{2-6}
             & Conv2D & $16-4$  & \lowres \(\rightarrow\) \midres  & \( 80.89\% \)          & $76.7\%$ \\
             & S4ND   & $16-4$  & \lowres \(\rightarrow\) \midres  & \( \mathbf{88.57}\% \) & $57.3\%$ \\
    \bottomrule
  \end{tabular}
  }
  \label{tab:progressive-resizing-results}
\end{table}

\begin{figure}[ht!]
\begin{minipage}{.6\linewidth}%
  \centering
  \includegraphics[width=\linewidth]{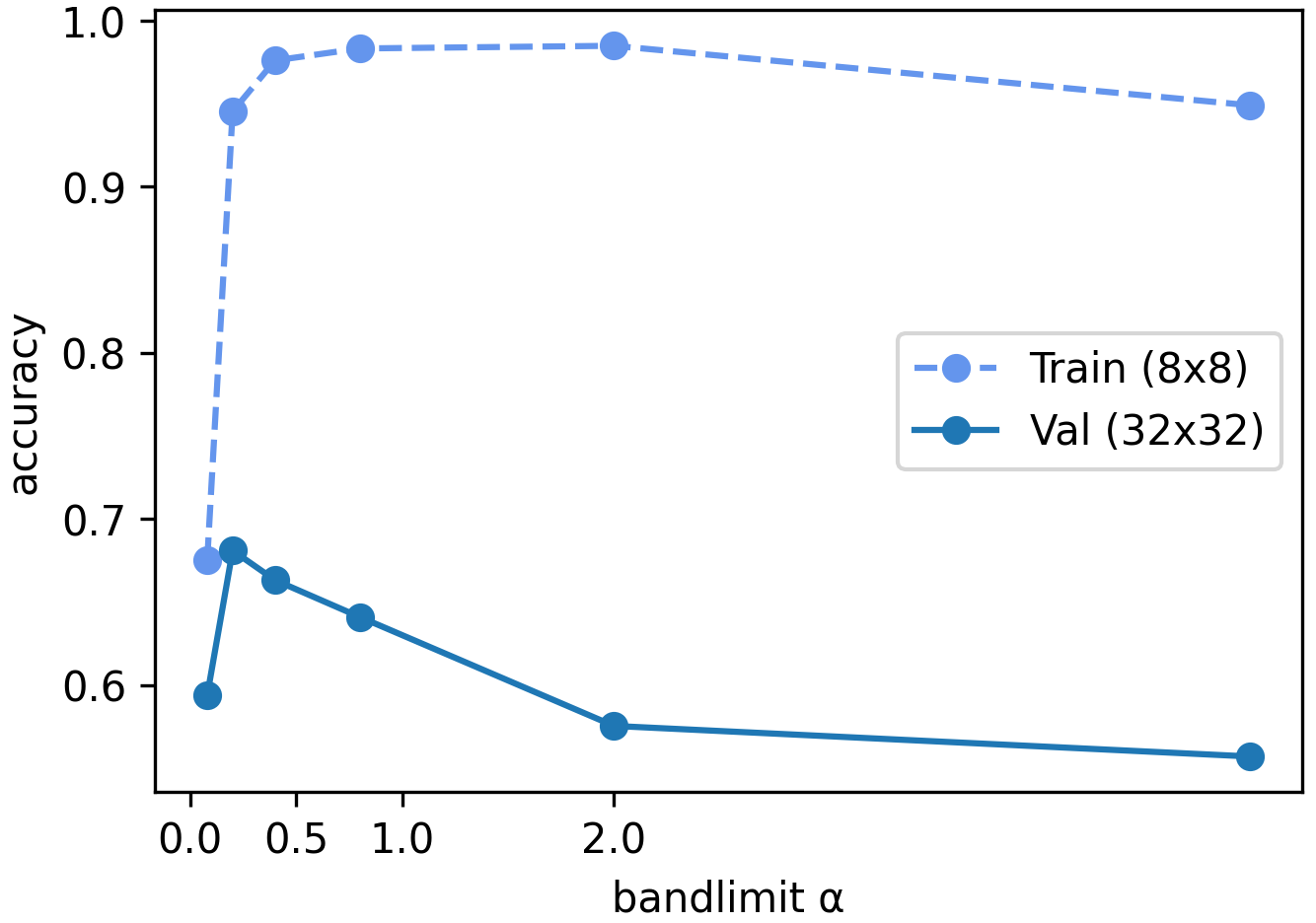}
  \captionsetup{type=figure}
  \caption{
      (\textbf{Bandlimiting ablation.})
      Zero-shot performance training at \( 8\times 8 \) and evaluating at \( 32\times 32 \). 
      Train performance is high for large enough values of \( \alpha \ge 0.5 \), but validation performance goes down as aliasing occurs in the kernel.
      Both train and validation drop for lower values of \( \alpha \), as it limits the expressivity of the kernel.
  }
  \label{fig:bandlimiting-param-ablation}
\end{minipage}
\hfill
\begin{minipage}{.35\linewidth}%
  \centering
  \begin{subfigure}{.49\linewidth}%
    \includegraphics[width=\linewidth]{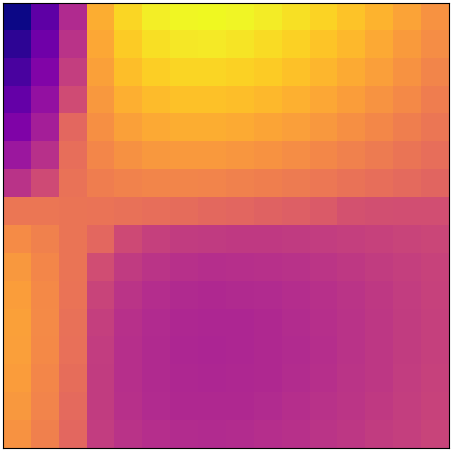}
    \caption{\( \alpha = 0.4 \) \\ \quad \( 16 \times 16 \) kernel}
  \end{subfigure}
  \hfill
  \begin{subfigure}{.49\linewidth}%
    \includegraphics[width=\textwidth]{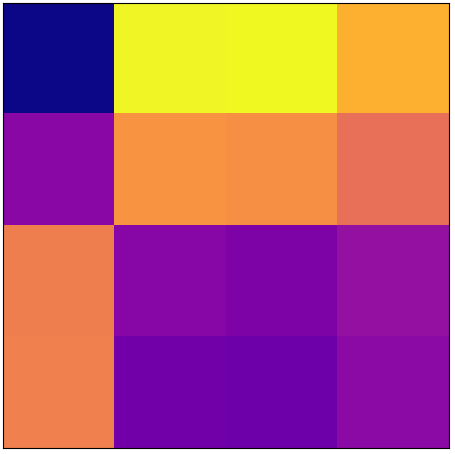}
    \caption{\( \alpha = 0.4 \) \\ \( 4 \times 4 \) kernel}
  \end{subfigure}
  \\
  \begin{subfigure}{.49\linewidth}%
    \includegraphics[width=\textwidth]{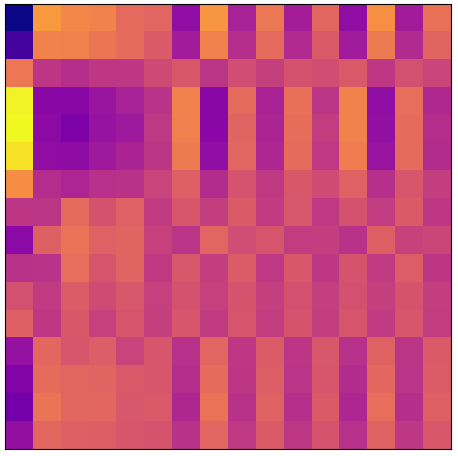}
    \caption{\( \alpha = \infty \) \\ \( 16 \times 16 \) kernel}
  \end{subfigure}
  \begin{subfigure}{.49\linewidth}%
    \includegraphics[width=\textwidth]{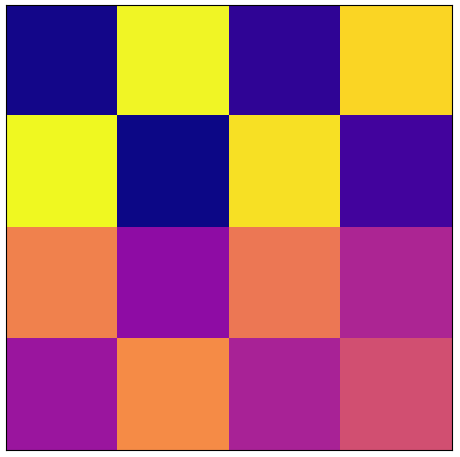}
    \caption{\( \alpha = \infty \) \\ \( 4 \times 4 \) kernel}
  \end{subfigure}
  \captionsetup{type=figure}
  \caption{
      (\textbf{Effect of bandlimiting on learned kernels.})
      Bandlimiting significantly increases the smoothness of the kernels when resizing resolutions.
  }
  \label{fig:bandlimiting-pictures}
\end{minipage}
\end{figure}

{\bf Effect of Bandlimiting.} 
Bandlimiting in \method\ is critical to generalization at different resolutions.
We analyze the effect of the bandlimiting parameter $\alpha$ on CIFAR-10 performance when doing zero-shot resolution change.
We additionally vary the choice of basis function $K_n(t)$ used in the SSM.
\cref{fig:bandlimiting-param-ablation} shows that zero-shot performance on \baseres\ degrades for larger values of $\alpha$, i.e. for cutoffs that do not remove high frequencies that violate the Nyquist cutoff. 
%
As we would expect, this holds regardless of the choice of basis function.  
In \cref{fig:bandlimiting-pictures}, we visualize learned kernels with and without bandlimiting, showing that bandlimiting improves smoothness. 
\cref{appendix:expt-details-continuous} includes additional experiments that analyze $\alpha$ on Celeb-A, and when doing progressively resized training.

\section{Discussion}
\label{sec:discussion}


\paragraph{Summary.} 

We introduced \method, a multidimensional extension of S4 that models visual data as continuous valued signals. 
\method\ is the first continuous model that matches SotA baselines on large-scale $1$D and $2$D image classification on ImageNet-1k, as well as outperforming a strong pretrained model in a $3$D video classification setting.
%
%
As a continuous-signal model, \method\ inherits useful properties that are absent from standard visual modeling approaches, such as zero-shot testing on unseen resolutions without a significant performance drop.
%

\paragraph{Limitations.} A limitation of \method\ is its training speed in high dimensions. 
In the $1$D image setting, \method-ViT has similar training speed to ViT; however, in the $2$D setting, the \method-ConvNeXt was $2\times$ slower than the baseline ConvNeXt.
We remark that vanilla local convolutions have been heavily optimized for years, and we expect that layers such as S4ND can be substantially sped up with more optimized implementations.
Our core computational primitives accounting for $65\%$ of our runtime (FFT, pointwise operations, inverse FFT) are all bottlenecked by reading from and writing to GPU memory~\citep{stvrelak2018performance}.
With a more optimized implementation that fuses these operations~\citep{dao2022flashattention} (i.e., loading the input once from GPU memory, perform all operations, then write result back to GPU memory), we expect to speed these up by 2-3$\times$. 
%
Further discussion can be found in \cref{appendix:discussion}.



\paragraph{Future work.} We presented a first step in using continuous-signal models in images and videos, and believe this opens the door to new capabilities and directions. For example, recent video benchmark datasets are significantly larger than the HMDB-51 dataset used in our experiments~\citep{kay2017kinetics, goyal2017something, damen2018scaling, monfort2019moments, sigurdsson2016hollywood}, and can be explored. In addition, we demonstrated capabilities in zero-shot resolution \textit{spatially}, but even less work has been done on zero-shot testing \textit{temporally}. This opens up an exciting new direction of work that could allow models to be agnostic to different video sampling rates as well, capable of testing on higher unseen sampling rates, or irregular (non-uniform) sampling rates. As models become larger and combine multiple modalities, \method\ shows strong promise in being able to better model underlying continuous-signals and create new capabilities across visual data, audio, time-series (e.g., wearable sensors, health data), and beyond.

\subsection*{Acknowledgments}
We thank Arjun Desai, Gautam Machiraju, Khaled Saab and Vishnu Sarukkai for helpful feedback on earlier drafts. This work was done with the support from the HAI-Google Cloud Credits Grant Program. We gratefully acknowledge the support of NIH under No. U54EB020405 (Mobilize), NSF under Nos. CCF1763315 (Beyond Sparsity), CCF1563078 (Volume to Velocity), and 1937301 (RTML); ONR under No. N000141712266 (Unifying Weak Supervision); ONR N00014-20-1-2480: Understanding and Applying Non-Euclidean Geometry in Machine Learning; N000142012275 (NEPTUNE); the Moore Foundation, NXP, Xilinx, LETI-CEA, Intel, IBM, Microsoft, NEC, Toshiba, TSMC, ARM, Hitachi, BASF, Accenture, Ericsson, Qualcomm, Analog Devices, the Okawa Foundation, American Family Insurance, Google Cloud, Salesforce, Total, the HAI-AWS Cloud Credits for Research program, the Stanford Data Science Initiative (SDSI), and members of the Stanford DAWN project: Facebook, Google, and VMWare. The Mobilize Center is a Biomedical Technology Resource Center, funded by the NIH National Institute of Biomedical Imaging and Bioengineering through Grant P41EB027060. The U.S. Government is authorized to reproduce and distribute reprints for Governmental purposes notwithstanding any copyright notation thereon. Any opinions, findings, and conclusions or recommendations expressed in this material are those of the authors and do not necessarily reflect the views, policies, or endorsements, either expressed or implied, of NIH, ONR, or the U.S. Government.

\bibliographystyle{plainnat}
\bibliography{biblio}

\newpage

\appendix

\section{Theory Details}
\label{sec:theory-details}

We formalize and generalize the notation of \cref{sec:method:s4nd} and prove the results.

For the remainder of this section, we fix a dimension \( D \) (e.g. \( D=2 \) for a 2-dimensional SSM).
The \( D \)-dimensional SSM will be a map from a function \( u : \mathbbm{R}^D \to \mathbbm{C} \) to \( y : \mathbbm{R}^D \to \mathbbm{C} \).

\begin{definition}[Indexing notation]%
  Let \( [d] \) denote the set \( \{1, 2, \dots, d\} \). Let \( [0] \) denote the empty set \( \{\} \).
  Given a subset \( I \subseteq [D] \), let \( -I \) denote its complement \( [D] \setminus I \).
  Let \( [-d] = -[d] = \{d+1, \dots, D\} \).
\end{definition}

\begin{definition}[Tensor products and contractions]%
  Given \( a \in \mathbbm{C}^{N_1} \) and \( b \in \mathbbm{C}^{N_2} \),
  let \( a \otimes b \in \mathbbm{C}^{N_1 \times N_2} \) be defined as \( (a \otimes b)_{n_1, n_2} = a_{n_1} b_{n_2} \).

  Given a tensor \( x \in \mathbbm{C}^{N_1 \times \dots \times N_D} \) and matrix \( \bm{A} \in \mathbbm{C}^{N_1 \times N_1} \),
  define \( \bm{A} \cdot^{(1)} x \in \mathbbm{C}^{N_1 \times \dots \times N_D} \) as
  \begin{align*}
    (\bm{A} \cdot^{(1)} x)_{n_1, \dots, n_D} = \sum_m \bm{A}_{n_1 m} x_{k, n_2, \dots, n_D}
    .
  \end{align*}
  which is simply a matrix multiplication over the first dimension.
  Let \( \cdot^{(\tau)} \) similarly denote matrix multiplication over any other axis.
\end{definition}

It will be easier to work directly from the convolutional definition of SSM (equation \eqref{eq:nd-kernel}).
We will then show equivalence to a PDE formulation (equation \eqref{eq:nd-ssm}).

\begin{definition}[Multidimensional SSM]%
  Let \( N_1, \dots, N_D \) be state sizes for each of the \( D \) dimensions.
  The \( D \)-dimensional SSM has parameters \(  \bm{A}^{(\tau)} \in \mathbbm{C}^{N^{(\tau)} \times N^{(\tau)}} \), \( \bm{B}^{(\tau)} \in \mathbbm{C}^{N^{(\tau)}} \), and \( \bm{C} \in \mathbbm{C}^{N^{(1)} \times \dots \times N^{(D)}} \) for each dimension \( \tau \in [D] \).

  and is defined by the map \( u \mapsto y \) given by
  \begin{align*}
    x(t) &= (u \ast \bigotimes_{\tau=1}^D e^{t^{(\tau)}\bm{A}^{(\tau)}} \bm{B}^{(\tau)})(t)
    \\
    y(t) &= \langle \bm{C}, x(t) \rangle
    .
  \end{align*}
  Note that at all times \( t = (t^{(1)}, \dots, t^{(D)}) \), the dimension of the state is \( x(t) \in \mathbbm{C}^{N_1 \times \dots \times N_D} \).
\end{definition}

Finally, it will be convenient for us to define ``partial bindings'' of \( u, x, y \) where one or more of the coordinates are fixed.

\begin{definition}[Partial binding of \( u \)]%
  Let \( I \subseteq [D] \) and let \( t^I = (t^{(I_1)}, \dots, t^{(I_{|I|})}) \) be an partial assignment of the time variables.
  Define \( u_I(t^I) \) to be the function \( \mathbbm{R}^{D-|I|} \mapsto \mathbbm{C} \) where the \( I \) indices are fixed to \( t^I \).

  For example, if \( D=3 \) and \( I = [1] \), then
  \begin{align*}
    u_I(t^I) = u(t^{(1)}, \cdot, \cdot)
  \end{align*}
  is a function from \( \mathbbm{R}^{2} \to \mathbbm{C} \) mapping \( (t^{(2)}, t^{(3)}) \mapsto u(t^{(1)}, t^{(2)}, t^{(3)}) \).
\end{definition}

\begin{definition}[Partial binding of \( x \)]%
  \begin{align*}
    x_I(t^I)(t^{-I}) &= (u_I(t^I) \ast \bigotimes_{\tau \in -I} e^{t^{(\tau)}\bm{A}^{(\tau)}} \bm{B}^{(\tau)})(t^{-I})
  \end{align*}
  Note that \( x_{[0]}() : \mathbbm{R}^D \mapsto \mathbbm{C}^{N_1 \times \dots \times N_D} \) coincides with the full state \( x \).
\end{definition}

The following more formal theorem shows the equivalence of this convolutional LTI system with a standard 1D SSM differential equation in each dimension.
\begin{theorem}%
  Given a partial state \( x_I(t^I) \) and a time variable \( t^{(\tau)} \) for \( \tau \not\in I \), let \( I' = I \cup \{\tau\} \).
  Then the partial derivatives satisfy
  \begin{align*}
    \frac{\partial x_I(t^{I})}{\partial t^{(\tau)}} (t^{-I})
    &= \bm{A}^{(\tau)} \cdot^{(\tau)} x_{I}(t^{I}) (t^{-I})
    \\&\qquad + \bm{B}^{(\tau)} \otimes x_{I'}(t^{I'}) (t^{-I'})
  \end{align*}
\end{theorem}
\begin{proof}%
  WLOG assume \( I = [d] \), since all notions are permutation-independent.
  We will consider differentiating the state \( x_I \) with respect to the time variable \( t^{(d+1)} \).
  The key fact is that differentiating a convolution \( \frac{d}{dt} (f \ast g) \) is equivalent to differentiating one of the operands \( f \ast (\frac{d}{dt} g) \).
  \begin{align*}
    \frac{\partial x_{[d]}(t^{[d]})}{\partial t^{(d+1)}} (t^{-[d]}) &= \frac{\partial}{\partial t^{(d+1)}} (u_I(t^I) \ast \bigotimes_{\tau \in -I} e^{t^{(\tau)}\bm{A}^{(\tau)}} \bm{B}^{(\tau)})(t^{-I})
    \\&= (u_I(t^I) \ast \frac{\partial}{\partial t^{(d+1)}} \bigotimes_{\tau \in -[d]} e^{t^{(\tau)}\bm{A}^{(\tau)}} \bm{B}^{(\tau)})(t^{-I})
    \\&= \left(u_I(t^I) \ast \left( \bm{A}^{(d+1)} e^{t^{(d+1)}\bm{A}^{(d+1)}} \bm{B}^{(d+1)} + \bm{B}^{(d+1)}\delta(t^{(d+1)}) \right) \bigotimes_{\tau \in -[d+1]} e^{t^{(\tau)}\bm{A}^{(\tau)}} \bm{B}^{(\tau)}\right) (t^{-I})
    \\&= \bm{A}^{(d+1)} \cdot^{(d+1)} \left(u_I(t^I) \ast \bigotimes_{\tau \in -[d]} e^{t^{(\tau)}\bm{A}^{(\tau)}} \bm{B}^{(\tau)}\right) (t^{-I})
    \\&\qquad + u_{[d+1]}(t^{[d+1]}) \bm{B}^{(d+1)} \bigotimes_{\tau \in -[d+1]} e^{t^{(\tau)}\bm{A}^{(\tau)}} \bm{B}^{(d+1)}
    \\&= \bm{A}^{(d+1)} \cdot^{(d+1)} x_{[d]}(t^{[d]}) (t^{-[d]})
    \\&\qquad + \bm{B}^{(d+1)} \otimes x_{[d+1]}(t^{[d+1]}) (t^{-[d+1]})
  \end{align*}
\end{proof}

The following corollary follows immediately from linearity of the convolution operator, allowing the order of convolution and inner product by \( \bm{C} \) to be switched.
\begin{corollary}%
  The output \( y \) is equivalent to \( y = K \ast u \) where
  \begin{align*}
    K(t) = \langle \bm{C}, \bigotimes_{\tau=1}^D e^{t^{(\tau)}\bm{A}^{(\tau)}} \bm{B}^{(\tau)} \rangle
  \end{align*}
\end{corollary}
This completes the proof of \cref{thm:nd-ssm}.

Finally, \cref{cor:nd-ssm-factored} is an immediate consequence of the Kronecker mixed-product identity
\begin{align*}
  (A \otimes B)(C \otimes D) = (AC) \otimes (BD)
  .
\end{align*}

\section{Experimental Details}
\label{appendix:expt-details}

We use PyTorch for all experiments, and build on the publicly available \href{https://github.com/HazyResearch/state-spaces}{S4} code.

\subsection{Image Classification}
\label{appendix:expt-details-image}

ImageNet training: all models were trained from scratch with no outside data using 8 Nvidia A100 GPUs. For both ViT and ConvNeXt experiments, we follow the procedure from T2T-ViT \citep{yuan2021tokens} and the original ConvNeXt \citep{Liu2022ACF}, respectively, with minor adjustments. Preprocessing and dataloading was done using the TIMM \citep{rw2019timm} library. In S4ND-ViT, we turn off weight decay and remove the class token prepending of the input sequence. For the ConvNeXt models, we add RepeatAug \citep{repeataug}, as well as reduce the batch size to 3840 for \method-ConvNeXt. 

\method\ specific settings include a bidirectional S4 kernel followed by Goel et al \citep{goel2022sashimi}, and a state dimension of 64 for the SSMs.

\begin{table}[ht]
      \small
      \caption{{(\bf Performance on image classification.)}
        ImageNet settings for ViT and ConvNeXt baseline models.
      }
        \centering
        \begin{tabular}{lcc}
            \toprule
            & \textsc{ViT} & \textsc{ConvNeXt}\\
            \midrule
            image size & $224^2$ &\ $224^2$ \\
            optimizer & AdamW & AdamW \\
            optimizer momentum & $\beta_1,\beta_2=0.9,0.999$ & $\beta_1,\beta_2=0.9,0.999$ \\
            weight init & trunc. normal (std=0.02) & trunc. normal (std=0.02) \\
            base learning rate & 0.001 & 0.004 \\
            weight decay & 0.05 & 0.05 \\
            dropout & None & None \\
            batch size & 4096 & 4096 \\
            training epochs & 300 & 300 \\
            learning rate schedule & cosine decay & cosine decay \\
            warmup epochs & 10 & 20 \\
            warmup schedule & linear & linear \\
            layer-wise lr decay \citep{bao2021beit,clark2020electra} & None & None \\
            randaugment \citep{cubuk2020randaugment} & (9,0.5,layers=2) & (9,0.5,layers=2) \\
            mixup \citep{zhang2017mixup} & 0.8 & 0.8 \\
            cutmix \citep{yun2019cutmix} & 1.0 & 1.0 \\
            repeataug \citep{repeataug} & None & 3 \\
            random erasing \citep{zhong2020random} & 0.25 & 0.25 \\
            label smoothing \citep{szegedy2016rethinking} & 0.1 & 0.1 \\
            stochastic depth \citep{huang2016deep} & 0.1 & 0.1 \\
            layer scale \citep{touvron2021going} & None & 1e-6 \\
            head init scale \citep{touvron2021going} & None & None \\
            exp.mov. avg (EMA) \citep{polyak1992ema} & 0.9999 & None \\
            \bottomrule
        \end{tabular}
        \label{tab:imagenet_hparams}
\end{table}

\subsection{Video Classification}
\label{appendix:expt-details-video}

HMDB-51 training: all models were trained from scratch on a single Nvidia A100 GPU. We use the Pytorchvideo library for data loading and minimal preprocessing of RGB frames only (no optical flow). During training, we randomly sample 2 second clips from each video. At validation and test time, 2 second clips are sampled uniformly to ensure that an entire video is seen. For each ConvNeXt 3D video model (baselines and \method\ versions), we fix the hyperparameters in \cref{tab:hmdb_hparams}, and do a sweep over the RandAugment\citep{cubuk2020randaugment} settings \texttt{num\_layers}\( =\{1,2\} \) and \texttt{magnitude}\( =\{3,5,7,9\} \).

\method\ specific settings were similar to image classification, a bidirectional S4 kernel and a state dimension of 64 for the SSMs.

\begin{table}[ht]
      \small
      \caption{{(\bf Performance on video classification.)}
        HMDB-51 settings for all 3D ConvNeXt models (baselines and \method\ models).
      }
        \centering
        \begin{tabular}{@{}lcc@{}}
            \toprule
            & \textsc{ConvNeXt 3D} \\
            \midrule
            image size & $224^2$ \\
            \# frames & 30 \\
            clip duration (sec) & 2 \\
            optimizer & AdamW \\
            optimizer momentum & $\beta_1,\beta_2=0.9,0.999$ \\
            weight init & trunc. normal (std=0.02) \\
            base learning rate & 0.0001 \\
            weight decay & 0.2  \\
            dropout (head) & 0.2 \\
            batch size & 64 \\
            training epochs & 50 \\
            learning rate schedule & cosine decay \\
            warmup epochs & 0 \\
            stochastic depth \citep{huang2016deep} & 0.2 \\
            layer scale \citep{touvron2021going} & 1e-6 \\
            head init scale \citep{touvron2021going} & None \\
            exp.mov. avg (EMA) \citep{polyak1992ema} & None \\
            \bottomrule
        \end{tabular}
        \label{tab:hmdb_hparams}
\end{table}

\subsection{Continuous Time Capabilities Experiments}
\label{appendix:expt-details-continuous}

We use the CIFAR-10 and Celeb-A datasets for all our experiments. Below, we include all experimental details including data processing, model training and hyperparameters, and evaluation details.

As noted in \cref{tab:dataset-resolutions}, we consider $3$ resolutions for each dataset, a \baseres\ resolution that is considered the standard resolution for that dataset, as well as two lower resolutions \lowres\ and \midres. Images at the lower resolutions are generated by taking an image at the \baseres\ resolution, and then downsampling using the \texttt{resize} function in \texttt{torchvision}, with bilinear interpolation and antialiasing turned on.

{\bf CIFAR-10.}  The base resolution is chosen to be $32 \times 32$, which is the resolution at which models are generally trained on this dataset. The lower resolutions are $16 \times 16$ and $8 \times 8$. 
We use no data augmentation for either zero-shot resolution change or progressive resizing experiments.
We train with standard cross-entropy loss on the task of $10$-way classification. 
For reporting results, we use standard accuracy over the $10$ classes.

We use a simple isotropic model backbone, identical to the one used by~\citet{gu2022efficiently} for their (1D) S4 model. The only difference is that we use either the \method\ or Conv2D layer in place of the S4 block, which ensures that the model can accept a batch of $2$D spatial inputs with channels. For results reported in the main body, we use a $6 \times 256$ architecture consisting of $6$ layers and a model dimension of $256$.

For all CIFAR-10 experiments, we train for $100$ epochs, use a base learning rate of $0.01$ and a weight decay of $0.03$.

{\bf Celeb-A.}
On Celeb-A, we consider $160 \times 160$ as the base resolution, and $128 \times 128$ and $64 \times 64$ as the lower resolutions.  
Images on Celeb-A are generally $218 \times 178$, and we run a $\mathrm{CenterCrop(178)} \rightarrow \mathrm{Resize(160)}$ transform in \texttt{torchvision} to resize the image.
We use no data augmentation for either zero-shot resolution change or progressive resizing experiments. 
We train with standard binary cross-entropy loss on the task of $40$-way multilabel attribute classification on Celeb-A. 
For reporting results, we use the multilabel binary accuracy, i.e. the average binary accuracy across all $40$ tasks.

Similar to our ImageNet experiments, we use a ConvNeXt model as the basic backbone for experiments on Celeb-A. However, use a particularly small model here, consisting of $4$ stages with depths $(3, 3, 3, 3)$ and corresponding model dimensions $(64, 128, 256, 512)$. For \method, we simply follow the same process as ImageNet, replacing all depthwise Conv2D layers with \method. A minor difference is that we use a global convolution in the stem downsampling for \method, rather than the smaller kernel size that we use for ImageNet. We use a drop path rate of $0.1$ for both \method\ and the Conv2D baseline in all experiments.

For all Celeb-A experiments, we train for $20$ epochs, use a base learning rate of $0.004$, and use automatic mixed precision for training.

{\bf Other fixed hyperparameters.} For all experiments, we use bidirectional S4 kernels following \citet{goel2022sashimi}, and use a state dimension of $64$ for the S4 SSMs. We use AdamW as the optimizer.

\subsubsection{Zero-Shot Resolution Change}
For \cref{tab:zero-shot-resolution-change-table}, we simply train on a single resolution among \baseres, \midres\ and \lowres, and then directly test at all higher resolutions. For model selection at a particular test resolution, we select the best performing model for that test resolution i.e. we checkpoint the model at the epoch where it performs best at the test resolution in question. Note that we refer to ``test resolutions" but report validation metrics in \cref{tab:zero-shot-resolution-change-table}, as is standard practice for additional experiments and ablations.

{\bf CIFAR-10.}
We use a batch size of $50$, and for both methods (and all resolutions), we use a cosine decay schedule for the learning rate with no restarts and a length of $100000$, and a linear warmup of $500$ steps.

For baseline Conv2D hyperparameters, we compare the performance of Conv2D with and without depthwise convolutions. Otherwise, all hyperparameters are fixed to the common values laid out in the previous section (and are identical to those used for \method).

For \method\ hyperparameters, we sweep the choice of initialization for the state space parameters $A, B$ in the S4 model among $\{\text{legs, fourier, random-linear, random-inv}\}$\footnote{These correspond to initializations of the model that allow the kernel to be expressed as a combination of different types of basis functions. For more details, we refer the reader to \citet{gu2020hippo,gu2022efficiently}.} and sweep bandlimit ($\alpha$) values among $\{0.05, 0.10, 0.20, 0.50, \infty\}$.

{\bf Celeb-A.} 
We use a batch size of $256$ when training on the lower resolutions, and $128$ when training at the \baseres\ resolution. For both methods (and all resolutions), we use a cosine decay schedule for the learning rate with no restarts and a length of $13000$ (batch size $256$) or $26000$ (batch size $128$), and a linear warmup of $500$ steps. 

For baseline ConvNeXt hyperparameters, we use a weight decay of $0.1$ after an initial sweep over weight decay values $\{0.05, 0.1, 0.2, 0.5\}$.

For \method\ hyperparameters, we sweep the choice of initialization for the state space parameters $A, B$ in the S4 model among $\{\text{legs, fourier, random-linear, random-inv}\}$, sweep bandlimit ($\alpha$) values among $\{0.05, 0.10, 0.20, 0.50, \infty\}$, and use a weight decay of $1.0$. The value of the weight decay was chosen based on an initial exploratory sweep over weight decay values $\{0.1, 0.2, 0.5, 1.0, 5.0\}$.

\subsubsection{Progressive Resizing}
For \cref{tab:progressive-resizing-results}, we train with a resizing schedule that progressively increases the resolution of the data. We report all metrics on validation sets, similar to our zero-shot experiments. We only use $2$ resizing stages in our experiments, as we found in preliminary experiments that $2+$ stages had little to no benefit. 

All hyperparameters and training details are identical to the zero-shot resolution change experiments, except for those we describe next. When training with multiple stages, we reset the cosine learning rate scheduler at the beginning of each stage (along with the linear warmup). The length of the scheduler is varied in proportion to the length of the stage, e.g. on CIFAR-10, for a stage $50$ epochs long, we change the length of the decay schedule from $100000$ steps to $50000$ steps (and similarly for Celeb-A). We always use a linear warmup of $500$ steps for each stage. Additionally, \cref{tab:progressive-resizing-results} denotes the length (in epochs) of each stage, as well as their training resolutions.

Another important detail is that we set the bandlimit parameter carefully for each stage. We found that bandlimiting in the first stage is critical to final performance, and without bandlimiting the performance on the \baseres\ resolution is substantially worse. Bandlimiting is generally not useful for a stage if it is training at the \baseres\ resolution i.e. at the resolution that we will be testing at (in our experiments, this is the case only for second stage training on CIFAR-10).

On CIFAR-10, we use a bandlimit $\alpha$ of $0.10$ or $0.20$ for the first stage, depending on whether the first stage was training at $8 \times 8$ (\lowres) or $16 \times 16$ (\midres) respectively. In the second stage, we always train at the \baseres\ resolution and use no bandlimiting, since as stated earlier, it has a negligible effect on performance. 

For Celeb-A, we set the bandlimit parameter to $0.1$ for both stages. Note that on Celeb-A, we do not train at the \baseres\ resolution at all, only training at \lowres\ in the first stage and \midres\ in the second stage.


\section{Discussion}
\label{appendix:discussion}

\paragraph{Runtime profiling and potential optimization.}
We profile the runtime of an S4ND layer on an A100 GPU, and plot the time taken by each operation in~\cref{fig:s4nd_profile}.
We see that the majority of the time (between 65\% and 80\%) are taken by the FFT, pointwise operation, and inverse FFT.
The higher-dimensional FFT is a very standard scientific primitive that should be substantially optimizable.

These operations are memory-bound, that is, the runtime is dominated the time to read/write to GPU memory.
There is currently no library support for fusing these steps, so for each of those steps the data has to be loaded from GPU memory, arithmetic operations are performed, then the result is written back to GPU memory.
One potential optimization is \emph{kernel fusion}: the input could be loaded once, all the arithmetic operations for all three steps are performed, then the final result is written back to GPU memory.
We expect that with library support for such optimization, the S4ND runtime can be reduced by 2-3$\times$.
\begin{figure}[ht]
  \centering
  \includegraphics[width=0.7\textwidth]{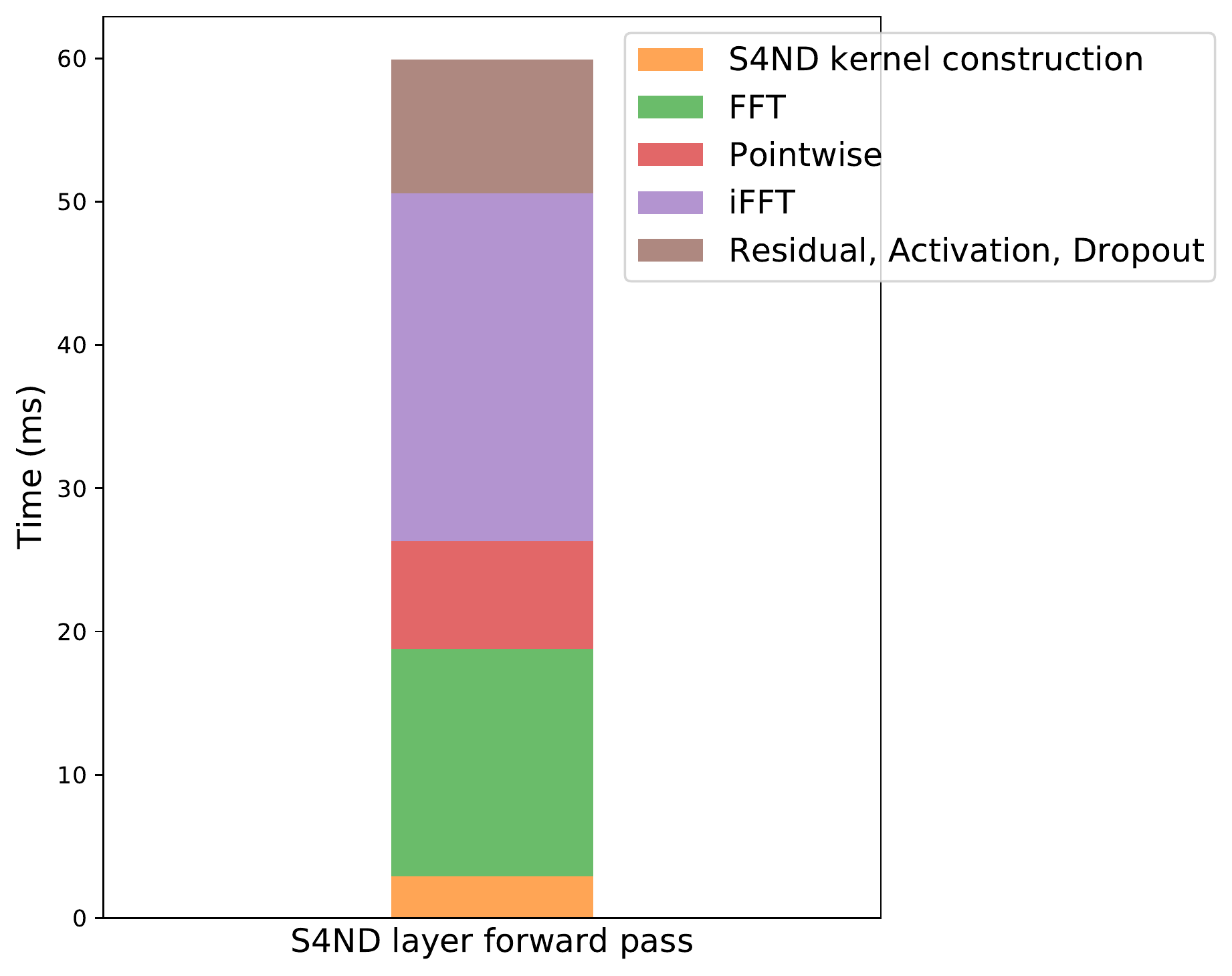}
  \caption{\label{fig:s4nd_profile}Timing breakdown of different steps in the S4ND forward pass, for a batch of 64 inputs, each of size 224 $\times$ 224, on an A100 GPU. FFT, pointwise operation, and inverse FFT take between 65\% and 80\% of the time.}
\end{figure}

\end{document}